\documentclass[conference,letterpaper]{IEEEtran}
\usepackage[top=0.75in, bottom=1.09in, left=0.625in, right=0.625in, columnsep=0.21in]{geometry}

\usepackage[utf8]{inputenc} 
\usepackage[T1]{fontenc}
\usepackage{url}
\usepackage{ifthen}
\usepackage{cite}
\usepackage[cmex10]{amsmath} % Use the [cmex10] option to ensure complicance
\usepackage{amssymb}
\usepackage{mathtools}

\usepackage{algorithm}
\usepackage{algpseudocode}

\usepackage[hidelinks]{hyperref}
\usepackage{xcolor}

\usepackage{amsthm}
\newtheorem{theorem}{Theorem}

\newtheorem{proposition}{Proposition}
\newtheorem{assumption}{Assumption}

\usepackage{flushend}

\usepackage{enumitem}

\usepackage{graphicx}
\usepackage{caption}
\usepackage{subcaption}
\newcommand{\algo}{CorN-DSGD}

\AtBeginDocument{
 \addtolength{\abovedisplayskip}{-3pt}
\addtolength{\abovedisplayshortskip}{-3pt}
\addtolength{\belowdisplayskip}{-3pt}
 \addtolength{\belowdisplayshortskip}{-3pt}
 \addtolength{\textfloatsep}{-3pt}
 \addtolength{\floatsep}{-3pt}
 \addtolength{\dblfloatsep}{-3pt}
 \addtolength{\intextsep}{-3pt}
\addtolength{\abovecaptionskip}{-2pt}
\addtolength{\belowcaptionskip}{-2pt}
}

\IEEEoverridecommandlockouts
\begin{document}
\title{Optimizing Privacy-Utility Trade-off in Decentralized Learning with Generalized Correlated Noise \vspace{-0.3cm}}
\author{%
    \IEEEauthorblockN{Angelo Rodio, Zheng Chen, and Erik G. Larsson}
    \IEEEauthorblockA{Dept. of Electrical Engineering (ISY), Link\"oping University, Sweden\\
                      Email: \{\emph{angelo.rodio, zheng.chen, erik.g.larsson}\}@liu.se\vspace{-0.9cm}}% 
    \thanks{This work was supported in part by the KAW foundation, ELLIIT, and the Swedish Research Council (VR).}
}

\maketitle
\bstctlcite{BSTcontrol}

% % page numbers
\thispagestyle{plain}
\pagestyle{plain}

% \add{in footnote: }
%%%%%%
%% Abstract: 
%% If your paper is eligible for the student paper award, please add
%% the comment "THIS PAPER IS ELIGIBLE FOR THE STUDENT PAPER
%% AWARD." as a first line in the abstract. 
%% For the final version of the accepted paper, please do not forget
%% to remove this comment!
%%

\begin{abstract}
Decentralized learning enables distributed agents to collaboratively train a shared machine learning model without a central server, through local computation and peer-to-peer communication.
Although each agent retains its dataset locally, sharing local models can still expose private information about the local training datasets to adversaries.
To mitigate privacy attacks, a common strategy is to inject random artificial noise at each agent before exchanging local models between neighbors. However, this often leads to utility degradation due to the negative effects of cumulated artificial noise on the learning algorithm. 
In this work, we introduce \emph{\algo}, a novel covariance-based framework for generating correlated privacy noise across agents, which unifies several state-of-the-art methods as special cases.
By leveraging network topology and mixing weights, \emph{\algo} optimizes the noise covariance to achieve network-wide noise cancellation.
Experimental results show that \emph{\algo} cancels more noise than existing pairwise correlation schemes, improving model performance under formal privacy guarantees.

\end{abstract}

\section{Introduction}
\label{sec:intro}
% !TEX root = main.tex

Training machine learning models traditionally involves centralizing datasets on a single server, which raises scalability and privacy risks~\cite{shokriMembershipInferenceAttacks2017}.
To address these issues, \emph{federated learning} allows distributed agents 
%(e.g., personal devices, organization, or computing nodes in a network) 
to retain their data locally while sharing only data-dependent computations (e.g., local gradients or models) with a central server~\cite{mcmahanCommunicationEfficientLearningDeep2017, konecnyFederatedLearningStrategies2017}. 
\emph{Decentralized learning} further removes the need for a central server by allowing agents to update their local models and mix them directly with neighbors according to predefined mixing weights~\cite{lianCanDecentralizedAlgorithms2017, assranStochasticGradientPush2019a, koloskovaUnifiedTheoryDecentralized2020a, nedicDistributedGradientMethods2020}.
Although federated and decentralized learning algorithms avoid sharing raw data over the network, local models may still expose sensitive information to adversaries
%---e.g., external eavesdroppers---
through membership inference or gradient inversion attacks~\cite{maTrustedAIMultiagent2023}.

A widely adopted theoretical concept for mitigating such threats is \emph{differential privacy (DP)}, which provides formal privacy guarantees by injecting random noise into data-dependent computations~\cite{dworkAlgorithmicFoundationsDifferential2014}.
While the privacy-utility trade-off has been extensively studied in the \emph{central DP (CDP)} setting--where a server injects noise on the central model~\cite{bassilyPrivateEmpiricalRisk2014, abadiDeepLearningDifferential2016, mironovRenyiDifferentialPrivacy2017}--adapting DP to decentralized settings proves substantially more challenging.
Under the \emph{local DP (LDP)} setting~\cite{kasiviswanathanWhatCanWe2011, duchiLocalPrivacyStatistical2013}, each agent individually adds noise to its local model before mixing, protecting its local dataset from adversaries. However, since each agent adds noise independently without coordination, the utility of LDP degrades in terms of model performance~\cite{duchiMinimaxOptimalProcedures2018}.

Recent works on differentially private decentralized learning aim to improve LDP utility by combining (i) a small amount of independent noise per agent, sufficient to ensure local privacy, and (ii) a large amount of \emph{correlated noise} across agents, designed to cancel during mixing~\cite{imtiazDistributedDifferentiallyPrivate2021, biswasLowCostPrivacyAwareDecentralized2024, allouah2024privacy}.
In principle, this cancellation design preserves local privacy while reducing the overall noise after mixing, thereby improving the accuracy of the learned model.
Among these works, \cite{imtiazDistributedDifferentiallyPrivate2021, biswasLowCostPrivacyAwareDecentralized2024} rely on \emph{zero‐sum correlation}, which is well-suited for addressing honest-but-curious internal neighbors, but ineffective against external eavesdroppers.
%as no privacy noise persists in the final model.
Meanwhile, \cite{allouah2024privacy} proposes \emph{pairwise-canceling correlated noise}, which restricts noise correlation to pairs of neighboring agents, 
overlooking how local models are mixed across the network and limiting the potential for broader 
correlation design.

In this paper, we argue that extending beyond pairwise, neighbor-only correlations can significantly improve the privacy-utility trade-off in decentralized learning. 
The key contributions are summarized as follows:
\begin{itemize}
 \item We present a novel, \emph{covariance-based} framework for the design of correlated noise across agents, that recovers state-of-the-art approaches as special cases;
\item We generalize existing analyses on the privacy-utility trade-off, showing that the optimal covariance matrix should account for the network topology and mixing weights. Building on this insight, we propose the \emph{\algo} algorithm, which optimizes the noise covariance to achieve network-wide noise cancellation.
\end{itemize}
Experiments with various privacy budgets and network connectivity levels show that \emph{\algo} achieves superior privacy-utility trade-offs compared to pairwise approaches, especially in weakly connected networks.

\section{Privacy-Preserving Decentralized Learning}
\label{sec:problem}

We consider $n$ agents, $\mathcal{V} = \{ 1, \dots, n \}$, each with a local dataset $D_i$, aiming to learn the parameters $x \in \mathbb{R}^{d}$ of a shared machine learning model, by minimizing the global objective:
\begin{align}
    F(x) \triangleq \frac{1}{n} \sum_{i=1}^{n} \left[ F_i(x) \triangleq \frac{1}{|D_i|} \sum_{\xi_i \in D_i} \ell(x,\xi_i) \right], 
    \label{eq:g_obj} 
\end{align}
where $\ell(x,\xi_i)$ is the loss of parameter $x$ on sample~$\xi_i \in D_i$ and $F_i(x)$ is the local objective, known only to agent~$i$.
Agents communicate over a network modeled by an undirected graph $\mathcal{G} = (\mathcal{V}, \mathcal{E})$, where an edge $(i,j) \in \mathcal{E}$ indicates that agents $i$ and $j$ are neighbors, i.e., they can directly communicate.
For simplicity, we present the case of scalar parameters ($d=1$), with straightforward extension to vector parameters.

% \egl{$F_i$ are differentiable, so that gradients are well defined?}

% Solving Problem~\eqref{eq:g_obj} requires communication among agents, as each local objective $F_i(x)$ depends on data held by agent~$i$. 

\subsection{Decentralized Stochastic Gradient Descent (DSGD)}

% Decentralized optimization algorithms like Decentralized Stochastic Gradient Descent (DSGD) allow to solve Problem~\eqref{eq:g_obj} in a decentralized fashion.

Problem~\eqref{eq:g_obj} is commonly solved via decentralized optimization algorithms like Decentralized Stochastic Gradient Descent (DSGD)~\cite{koloskovaUnifiedTheoryDecentralized2020a}. 
Each iteration~$t \in \{ 1, \dots, T \}$ involves two steps:

\begin{subequations}
\emph{a) Local step:} Each agent $i$ computes a stochastic gradient $g_i^{(t)} = \ell'(x_i^{(t)},\xi_i^{(t)})$, where $\xi_i^{(t)}$ is a data point sampled from agent~$i$'s dataset~$D_i$, and then updates its parameter: 
% $x_i^{(t+\frac{1}{2})} = x_i^{(t)} - \eta_t g_i^{(t)}$,
\begin{align}
    \textstyle
    x_i^{(t+\frac{1}{2})} = x_i^{(t)} - \eta_t g_i^{(t)},
\end{align}
where $\eta_t$ is the step-size. 

\emph{b) Mixing step:} Agents exchange parameters with neighbors over the network and compute a weighted average:
\begin{align}
    \textstyle
    {x}_i^{(t+1)} = \sum_{j=1}^n w_{ij} x_j^{(t+\frac{1}{2})},
\end{align}
where $w_{ij} \triangleq [\mathbf{W}]_{ij}$ are mixing weights defined by the mixing matrix $\mathbf{W} \in \mathbb{R}^{n \times n}$, with $w_{ij}=0$ whenever $(i,j)\notin\mathcal{E}$.
Convergence conditions for DSGD require $\mathbf{W}$ to be doubly stochastic ($\mathbf{W} \mathbf{1} = \mathbf{1}$ and $\mathbf{1}^\top \mathbf{W} = \mathbf{1}^\top$) and the second largest eigenvalue $\lambda_2(\mathbf{W}^\top \mathbf{W})$ strictly below one~\cite[Section II.B]{boydRandomizedGossipAlgorithms2006}.
\label{eq:DSGD}
\end{subequations}

\subsection{Differentially-Private DSGD}

In decentralized learning, we aim to protect agents’ local models from potential adversaries, which can be either \textit{external eavesdroppers} or \textit{curious internal agents}. For clarity of exposition, we address external eavesdroppers in the remainder of the paper and defer the discussion on honest-but-curious agents to Section~\ref{subsec:hbc}. Our privacy goal is to prevent the adversary from inferring whether any specific \emph{agent}—and by extension, its entire local dataset—is participating in the training process.\footnote{An alternative, weaker privacy objective is to protect each individual data sample rather than the entire dataset of an agent. Our framework extends to this sample-level DP setting in a straightforward manner.}

% the only public information is the
% stream of encrypted or plaintext model updates exchanged over the
% network.  We assume an honest‑but‑curious observer that records every
% message and aims to infer whether a specific \emph{agent} (not merely a
% single sample) took part in training.

% % Decentralized learning algorithms, like DSGD, expose private agents' data through shared local models. Adversaries can be external attackers, who eavesdrop communications, or honest-but-curious agents, who follow the protocol honestly yet attempt to infer other agents' data.
% %The DP framework provides formal theoretical guarantees on the privacy of the agents' local data. 
% In a decentralized setting, we focus on protecting the agents' local datasets against external adversaries who can eavesdrop on the communicated local model updates during the mixing step.
% \zheng{Be more specific about the threat model, highlight the difference between sample-level DP and agent-level DP. Hide participation of a single agent instead of the existence of a single data point. But the notion can be extended to sample-level DP as well. }

% \begin{definition}[Agent-level DP]
    Two sets of datasets $D = \{D_1, D_2, \dots, D_n\}$ and $D' = \{D'_1, D'_2, \dots, D'_n\}$ are \emph{agent-level neighbors} if they differ by exactly one agent's data ($D_j \neq D'_j$ for some $j$, and $D_i = D'_i$ for all $i\neq j$). A randomized mechanism $M(\cdot)$ is \emph{agent-level $(\varepsilon,\delta)$-DP} if, for all neighboring datasets $D, D'$ and all subsets $S$ of outputs, we have:
\begin{align}
        \mathrm{Pr}\bigl[M(D) \in S \bigr] 
        \leq
        e^\varepsilon \mathrm{Pr}\bigl[M(D') \in S \bigr] 
        + \delta.
\end{align}
% \end{definition}

% Two main threat models are considered, differing in trust assumptions and where noise is added. 
% In CDP, a trusted aggregator (e.g., a central server) collects and averages gradients from all agents ($\bar{g}^{(t)} = \frac{1}{n}\sum_{i=1}^{n}g_i^{(t)}$), clips the averaged gradient ($\hat{g}^{(t)} = \min\{ 1, \frac{C}{|\bar{g}^{(t)}|} \} \cdot \bar{g}^{(t)}$), adds noise \emph{centrally} ($\tilde{g}^{(t)} = \hat{g}^{(t)} + v^{(t)}$), and updates the central model ($\tilde{x}^{(t+1)} = \tilde{x}^{(t)} - \eta \tilde{g}^{(t)}$). Specifically, the CDP Gaussian mechanism adds noise $v^{(t)} \sim \mathcal{N}(0,\sigma_{\text{cdp}}^2)$, such that, for $\sigma_{\text{cdp}}^2 \geq \frac{2 C^2 \ln(1.25/\delta)}{\varepsilon^2}$, the central model $\tilde{x}^{(t+1)}$ is \emph{agent-level $(\varepsilon,\delta)$-CDP}~\cite{dworkAlgorithmicFoundationsDifferential2014,abadiDeepLearningDifferential2016}.

In the absence of a central server, privacy guarantees must be enforced locally by each agent. A classical LDP mechanism is presented in~\cite{kasiviswanathanWhatCanWe2011}. At each iteration $t$, 
each agent clips its own local gradient: $\hat{g}_i^{(t)} = \min\{ 1, \frac{C}{|g_i^{(t)}|}\} \cdot g_i^{(t)}$. It then adds noise \emph{locally} and \emph{independently}: $\tilde{g}_i^{(t)} = \hat{g}_i^{(t)} + v_i^{(t)}$. Finally, it applies the mixing step: $\tilde{x}_i^{(t+1)} = \sum_{j=1}^{n} w_{ij}(\tilde{x}_j^{(t)}-\eta_t \tilde{g}_j^{(t)})$. 
Specifically, the LDP Gaussian mechanism adds noise $v_i^{(t)} \sim \mathcal{N}(0,\sigma_{\text{ldp}}^2)$, such that, for sufficiently large noise variance $\sigma_{\text{ldp}}^2$, all local models $\{ \tilde{x}_i^{(t+1)} \}_{i=1}^{n}$ are \emph{agent-level $(\varepsilon,\delta)$-LDP}~\cite{duchiLocalPrivacyStatistical2013}. 

Despite its strong privacy guarantees, this LDP mechanism suffers from utility degradation,
%Since each agent adds noise independently, the effective noise in the local models does not cancel during mixing. Instead, 
since the added independent privacy noise accumulates over the network, leading to higher variance in the mixed models, slower convergence, and reduced model performance~\cite{duchiMinimaxOptimalProcedures2018}.
% In decentralized learning, LDP is the reference threat model. However, for the LDP mechanism to match the noise variance on the central model as the CDP mechanism ($\mathrm{Var}(v^{(t)}) = \mathrm{Var}(\frac{1}{n}\sum_{i=1}^{n} v_i^{(t)})$), each agent must locally add noise with variance $\sigma_{\text{ldp}}^2 = n \sigma_{\text{cdp}}^2$. In practice, such large $\sigma_{\text{ldp}}^2$ degrades the utility of the LDP mechanism in terms of model performance. 
%\subsection{Pairwise-Canceling Correlated Noise Across Agents}
%\subsection{Improving Utility with LDP Guarantees using Correlated Noise Across Agents}
Recognizing the utility degradation inherent to LDP, prior work~\cite{allouah2024privacy} proposed pairwise-canceling correlated noise across agents. In \emph{DECOR}~\cite{allouah2024privacy}, each agent adds noise $v_i^{(t)} = u_i^{(t)} + \sum_{j \in \mathcal{N}_i} v_{ij}^{(t)}$, where $u_i^{(t)} \sim \mathcal{N}(0,\sigma_{\text{pair}}^2)$ is independent local noise, and each pair of neighbors $(i,j)\in\mathcal{E}$ shares a correlated noise term $v_{ij}^{(t)} = -v_{ji}^{(t)} \sim \mathcal{N}(0,\sigma_{\text{cor}}^2)$ that cancels in the mixing~step. Distinct pairs of neighbors produce mutually uncorrelated components,
so that $v_{ij}^{(t)}$ is independent of $v_{k\ell}^{(t)}$ whenever $\{i,j\}\cap\{k,\ell\}=\emptyset$. 
However, DECOR restricts noise correlations to neighboring agents and offers fewer degrees of freedom in noise design, particularly in sparse networks.

This paper introduces a novel \emph{covariance-based} framework for generating correlated noise across agents, showing how the correlation structure should be designed by taking into account the effects of network topology and mixing weights.

\section{A Covariance-Based Framework for Generating Correlated Noise Across Agents}

We reinterpret the differentially private DSGD framework from a noise‐covariance perspective. Instead of drawing noise components $\{v_i^{(t)}\}_{i=1}^n$ independently or from pairwise correlations, we allow agents to sample from a multivariate Gaussian:
\begin{align}
    \mathbf{v}^{(t)} \;=\; \bigl[v_1^{(t)}, \dots, v_n^{(t)}\bigr]^\top \;\sim\; \mathcal{N}\!\bigl(\mathbf{0}, \mathbf{R}\bigr),
    \label{eq:R_noise}
\end{align}
where $\mathbf{R} \in \mathbb{R}^{n \times n}$ is an arbitrary covariance matrix. By defining $\tilde{\mathbf{x}}^{(t)} \triangleq [\tilde{x}_1^{(t)}, \dots, \tilde{x}_n^{(t)}]^\top$ and $\tilde{\mathbf{g}}^{(t)} = [\tilde{g}_1^{(t)}, \ldots, \tilde{g}_n^{(t)}]^\top$,
the perturbed model update at iteration $t$ becomes:
\begin{align}
    \tilde{\mathbf{x}}^{(t+1)} = \mathbf{W} \left( \tilde{\mathbf{x}}^{(t)} - \eta_t \tilde{\mathbf{g}}^{(t)} \right). 
    \label{eq:update_dp_DSGD}
\end{align}

In practice, each agent can generate $\mathbf{v}^{(t)}$ locally by sharing~$\mathbf{R}$ and a random seed $s$, drawing $\mathbf{s}^{(t)} \sim \mathcal{N}(\mathbf{0},\mathbf{I}_n)$ independently using the seed $s$, and locally computing $\mathbf{v}^{(t)} = \mathbf{R}^{1/2}\mathbf{s}^{(t)}$.

This covariance-based approach offers several advantages: \emph{(i)} it allows for more flexible correlation structures with noise cancellation beyond immediate neighbors; \emph{(ii)} it is directly applicable to directed graphs, unlike pairwise correlation; and \emph{(iii)} it recovers state-of-the-art approaches as special cases:
\begin{enumerate}
    \item \textit{LDP}. With independent noise, the covariance matrix is $\mathbf{R} = \sigma_{\text{ldp}}^2 \mathbf{I}_n$, where $\mathbf{I}_n \in \mathbb{R}^{n \times n}$ is the identity matrix;
    \item \textit{DECOR}. The independent noise term $u_i^{(t)} \sim \mathcal{N}(0,\sigma_{\text{pair}}^2)$ contributes $\sigma_{\text{pair}}^2\mathbf{I}_n$ to $\mathbf{R}$. In addition, each edge $(i,j)\in\mathcal{E}$ adds a pairwise correlated component $v_{ij}^{(t)} = -v_{ji}^{(t)}$ with variance $\sigma_{\text{cor}}^2$. In matrix form, each edge contributes $\sigma_{\text{cor}}^2(\mathbf{e}_i-\mathbf{e}_j)(\mathbf{e}_i-\mathbf{e}_j)^\top$, where $\mathbf{e}_i$ is the $i$-th standard basis vector, and summing over all edges yields $\sigma_{\text{cor}}^2\mathbf{L}$, where $\mathbf{L} = \sum_{(i,j)\in\mathcal{E}}
  (\mathbf{e}_i - \mathbf{e}_j)(\mathbf{e}_i - \mathbf{e}_j)^\top$ is the undirected graph Laplacian. The total covariance is $\mathbf{R} = \sigma_{\text{pair}}^2\mathbf{I}_n + \sigma_{\text{cor}}^2\mathbf{L}$.
\end{enumerate}

% The covariance matrix $\mathbf{R}$ controls both the model update variance (that affects convergence rate or utility) and the privacy level (in terms of $(\varepsilon,\delta)$-LDP guarantees).
Constraining~$\mathbf{R}$ to the graph Laplacian $\sigma_{\text{cor}}^2\mathbf{L}$ limits noise correlation, particularly in sparse graphs. For instance, in a star topology, $\mathbf{R} = \sigma_{\text{pair}}^2\mathbf{I}_n + \sigma_{\text{cor}}^2\mathbf{L}$ forces $n(n-1) - 2(n-1)$ covariance entries to zero, and overlooks all leaf‑to‑leaf noise cancellations, which a general $\mathbf{R} \succeq 0$ can instead exploit. 

\section{Privacy-Utility Analysis}
\label{sec:analysis}

% We analyze utility and privacy of our general differentially-private DSGD framework quantifying how correlated noise affects both the convergence and the $(\varepsilon,\delta)$-DP guarantees. 

The covariance matrix $\mathbf{R}$ affects both utility (i.e., convergence rate) and privacy (i.e., $(\varepsilon,\delta)$-LDP guarantees).

% convergence (utility) performance and $(\varepsilon,\delta)$-LDP (privacy) level. 

% Our analysis shows the interplay between noise covariance $\mathbf{R}$ and mixing matrix $\mathbf{W}$, yielding new insights into optimal noise design strategies.

% We analyze the privacy–utility trade-off of differentially private DSGD (Algorithm~\ref{alg:vector_notation}) quantifying how the correlated noise affects both utility (convergence) and privacy (DP guarantees).

\subsection{Utility Analysis}

We introduce a virtual sequence $\hat{\mathbf{x}}^{(t+1)} = \mathbf{W} (\tilde{\mathbf{x}}^{(t)} - \eta_t \hat{\mathbf{g}}^{(t)})$,
which begins at the noisy iterate \(\tilde{\mathbf{x}}^{(t)}\) but applies a noise-free DSGD step.
This construction isolates the error introduced by the privacy noise 
and simplifies the analysis in~\cite[Thm.~2]{allouah2024privacy}.

% We now analyze the convergence behavior of our covariance-based framework by comparing the noisy iterate \(\tilde{\mathbf{x}}^{(t)}\) to a noise-free auxiliary sequence. For clarity, we assume:
% Define a virtual, noise-free sequence \(\mathbf{x}^{(t+1)}\) that starts from the noisy point \(\tilde{\mathbf{x}}^{(t)}\) but applies a standard DSGD step:
% \begin{equation}
% \label{eq:noise_free_update}
%     \mathbf{x}^{(t+1)} \;=\; \mathbf{W}\Bigl(\tilde{\mathbf{x}}^{(t)} - \eta_t \,\mathbf{g}^{(t)}\Bigr).
% \end{equation}
% This construction allows us to decompose the error introduced by the privacy noise from the baseline DSGD convergence.

% Our analysis only relies on the following assumption.
\begin{assumption}
\label{assump:noise_indep}
At iteration \(t\), the privacy noise \(\mathbf{v}^{(t)}\) is independent of the noise-free state \(\hat{\mathbf{x}}^{(t+1)}\) and prior states 
\(\mathcal{F}^{(t)}=\{(\mathbf{v}^{(1)},\hat{\mathbf{x}}^{(2)}),\dots,(\mathbf{v}^{(t-1)},\hat{\mathbf{x}}^{(t)})\}\):
$
  \mathbb{E}[\mathbf{v}^{(t)}|
  {\hat{\mathbf{x}}^{(t+1)}}, \mathcal{F}^{(t)}] 
  = \mathbf{0}.
$
\end{assumption}

\begin{theorem}
\label{thm:convergence}
Under Assumption~\ref{assump:noise_indep}, the expected squared error between consecutive iterates of our \emph{covariance-based} framework satisfies:
\begin{align}
    &\mathbb{E}\Bigl[\bigl\|\tilde{\mathbf{x}}^{(t+1)} - \tilde{\mathbf{x}}^{(t)}\bigr\|^2
    \;\Big|\;\mathcal{F}^{(t)}\Bigr] \notag \\
    &=\;
    \underbrace{
    \mathbb{E}\Bigl[\bigl\|\hat{\mathbf{x}}^{(t+1)} - \tilde{\mathbf{x}}^{(t)}\bigr\|^2
    \;\Big|\;\mathcal{F}^{(t)}\Bigr]
    }_{\substack{\text{Noise-free DSGD error} \\ \text{with clipped gradients}}}
    \;+\;
    \underbrace{
    \eta_t^2 \,\mathrm{Tr}\!\bigl(\mathbf{W}\mathbf{R}\mathbf{W}^\top\bigr)
    }_{\text{Privacy-noise variance}}.
    \label{eq:update_error}
\end{align}
% \footnotetext{The convergence behavior of DSGD with clipped gradients has been extensively studied (e.g.,~\cite{zhangImprovedAnalysisClipping2020, liSoteriaFLUnifiedFramework2022, liConvergencePrivacyDecentralized2025}), and we do not address it in this work.}
% \begin{align}
%     &\mathbb{E}_{\mathbf{v}^{(t)}, \mathbf{x}^{(t+1)} \mid \mathcal{F}^{(t)}} 
%     \left\| \tilde{\mathbf{x}}^{(t+1)} - \tilde{\mathbf{x}}^{(t)} \right\|_2^2
%     = \notag \\
%     &~= \underbrace{
%     \mathbb{E}_{\mathbf{x}^{(t+1)} \mid \mathcal{F}^{(t)}} 
%     \left\| \mathbf{x}^{(t+1)} - \tilde{\mathbf{x}}^{(t)} \right\|_2^2
%     }_{\text{Noise-free DSGD error}}
%     + \underbrace{\eta_t^2 \operatorname{Tr}(\mathbf{W} \mathbf{R} \mathbf{W}^\top) \vphantom{\mathbb{E}_{\mathbf{x}^{(t+1)} \mid \mathcal{F}^{(t)}} 
%     \left\| \mathbf{x}^{(t+1)} - \mathbf{x}^\star \right\|_2^2}}_{\text{Variance of privacy noise}}.
% \end{align} 
\end{theorem}

% \begin{remark}
% Under standard convex assumptions, \emph{\algo} converges at the baseline DSGD rate with clipped gradients plus an additive noise-dependent term: $\mathcal{O}(\frac{\mathrm{Tr}(\mathbf{W}\mathbf{R}\mathbf{W}^\top)}{\sqrt{T}})$.
% \end{remark}
% \egl{This is vague. What is "standard convex"? Convex, or strongly convex, or something else? Smoothness not required? What kind
% of convergence, is it $x$ that converges, or the objective, and in what sense?
% Also, is this remark part of the theorem? It feels odd to have it inbetween of the theorem and its proof.}
% \zc{This remark says almost the same thing as the text "Theorem 1 shows...". It may look repetitive. I suggest removing this remark.}

\begin{proof}[Proof Sketch] 
First, decompose \(\|\tilde{\mathbf{x}}^{(t+1)} - \tilde{\mathbf{x}}^{(t)}\|^2\) into the sum of \(\|\hat{\mathbf{x}}^{(t+1)} - \tilde{\mathbf{x}}^{(t)}\|^2\) and \(\|\tilde{\mathbf{x}}^{(t+1)} - \hat{\mathbf{x}}^{(t+1)}\|^2\). Taking conditional expectation, the cross‐term vanishes by Assumption~\ref{assump:noise_indep}.
Next, observe \(\|\tilde{\mathbf{x}}^{(t+1)} - \hat{\mathbf{x}}^{(t+1)}\|^2 = \eta_t^2\|\mathbf{W}\mathbf{v}^{(t)}\|^2\). Finally, 
\(\mathbf{R}=\mathbb{E}[\mathbf{v}^{(t)}(\mathbf{v}^{(t)})^\top]\)
yields the variance term 
\(\eta_t^2\,\mathrm{Tr}(\mathbf{W}\mathbf{R}\mathbf{W}^\top)\).
\end{proof}

Theorem~\ref{thm:convergence} shows that each iteration of our covariance-based framework incurs a baseline DSGD error (independent of \(\mathbf{R}\)) plus a variance term proportional to \(\mathrm{Tr}(\mathbf{W}\mathbf{R}\mathbf{W}^\top)\).
% \footnote{Due to space limits, we omit here the full characterization of global convergence guarantees after $T$ iterations.} 

Our analysis captures the interplay between \(\mathbf{W}\) and \(\mathbf{R}\), explaining how correlated noise propagates and potentially cancels across agents.
Prior analyses overlooked this mutual dependency, which is central to our \emph{covariance‐based} view.

Additionally, the variance \(\mathrm{Tr}(\mathbf{W}\mathbf{R}\mathbf{W}^\top)\) is linear in $\mathbf{R}$: smaller $[\mathbf{R}]_{ii}$ components speed-up convergence. This implies that the optimal correlation structure should depend on the underlying network topology and mixing weights.

\subsection{Privacy Analysis}

We provide $(\varepsilon,\delta)$-LDP guarantees for our framework with arbitrary $\mathbf{R}\succ 0$. Theorem~\ref{thm:privacy_dp} bounds $\varepsilon$ in terms of number of iterations~\(T\), clipping threshold~\(C\), and inverse covariance~\(\mathbf{R}^{-1}\).

\begin{theorem}
\label{thm:privacy_dp}
Our \emph{covariance-based} framework is ($\varepsilon,\delta$)-LDP:
\begin{align}
    \textstyle
    \varepsilon \leq 2 C^2 T \max\limits_{i \in \mathcal{V}}[\mathbf{R}^{-1}]_{ii} + 2 C \sqrt{2T \log(\frac{1}{\delta}) \max\limits_{i \in \mathcal{V}}[\mathbf{R}^{-1}]_{ii}}.
    \label{eq:thm:privacy_dp}
\end{align}
\end{theorem}

\begin{proof}[Proof Sketch]
We adapt existing Rényi DP (RDP) arguments from~\cite{mironovRenyiDifferentialPrivacy2017, allouah2024privacy} to arbitrary $\mathbf{R}\succ 0$.  At each iteration, our framework satisfies $(\alpha,\alpha\varepsilon_{\text{step}})$-RDP with $\varepsilon_{\text{step}} = 2 C^2 \max_{i \in \mathcal{V}} [\mathbf{R}^{-1}]_{ii}$~\cite[Thm.~6]{allouah2024privacy}.
Since RDP composes additively over $T$ iterations~\cite[Prop.~1]{mironovRenyiDifferentialPrivacy2017}, the final iterate $\tilde{\mathbf{x}}^{(T)}$ satisfies ($\alpha, T \alpha \varepsilon_\text{step}$)-RDP. Converting RDP to $(\varepsilon,\delta)$-DP introduces an additional term $\ln(1/\delta)/(\alpha-1)$~\cite[Prop.~3]{mironovRenyiDifferentialPrivacy2017}. Optimizing the expression over $\alpha>1$ concludes the proof.
\end{proof}

Theorem~\ref{thm:privacy_dp} directly relates the diagonal elements of $\mathbf{R}^{-1}$, specifically $\max_{i\in\mathcal{V}} [\mathbf{R}^{-1}]_{ii}$, to the privacy budget~$\varepsilon$. This result immediately recovers well-known LDP bounds (e.g., \cite[Prop.~7]{mironovRenyiDifferentialPrivacy2017}) when $\mathbf{R}=\sigma_{\text{ldp}}^2 \mathbf{I}_n$: larger $[\mathbf{R}]_{ii}$ (i.e., larger variance of privacy noise) implies smaller $[\mathbf{R}^{-1}]_{ii}$, thus smaller $\varepsilon$ (that is, strengthened $(\varepsilon,\delta)$-LDP guarantees).

Invertibility (\(\mathbf{R}\succ 0\)) is a necessary condition for LDP guarantees under correlated noise: if $\mathbf{R}$ were singular, linear combinations of noise could cancel completely across agents, directly exposing their local models after the mixing step.

\section{The ``\algo'' Algorithm}
\label{sec:algo}

\begin{figure*}[tbp]
    \centering
    \hspace{-4.4cm}
    \includegraphics[width=0.5\linewidth]{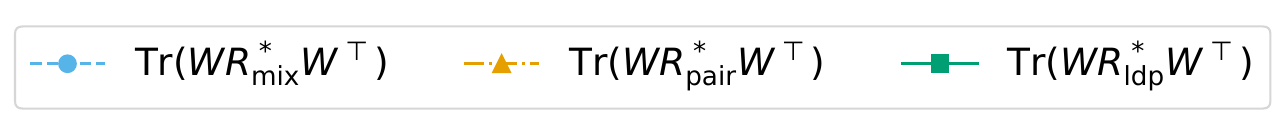}

    % Left: three subfigures
    \begin{minipage}[t]{0.74\textwidth}
        \centering
        \begin{subfigure}[b]{0.32\linewidth}
            \centering
            \includegraphics[width=\linewidth]{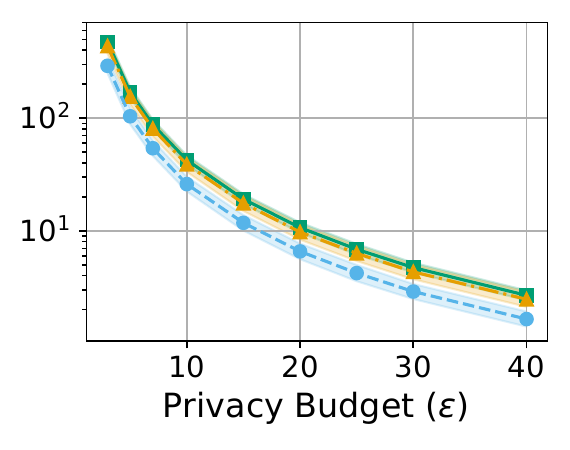}
            \caption{$n=20$, $p=0.5$.}
            \label{fig:sub1}
        \end{subfigure}
        \hfill
        \begin{subfigure}[b]{0.32\linewidth}
            \centering
            \includegraphics[width=\linewidth]{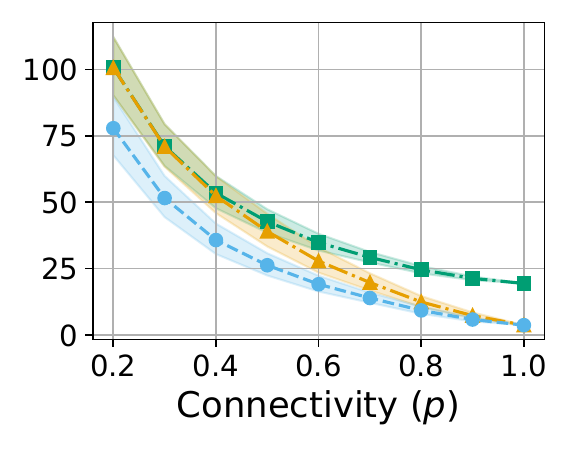}
            \caption{$n=20$, $\varepsilon=10$.}
            \label{fig:sub2}
        \end{subfigure}
        \hfill
        \begin{subfigure}[b]{0.32\linewidth}
            \centering
            \includegraphics[width=\linewidth]{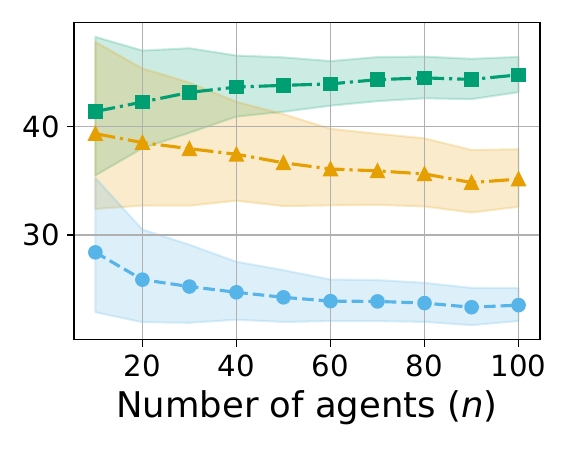}
            \caption{$p=0.5$, $\varepsilon=10$.}
            \label{fig:sub3}
        \end{subfigure}
        \caption{$\mathrm{Tr}(\mathbf{W}\mathbf{R}_{\text{mix}}\mathbf{W}^\top)$, $\mathrm{Tr}(\mathbf{W}\mathbf{R}_{\text{pair}}\mathbf{W}^\top)$, $\mathrm{Tr}(\mathbf{W}\mathbf{R}_{\text{ldp}}\mathbf{W}^\top)$ vs. varying: (a)~privacy budget~$\varepsilon$ at $p=0.5$, (b)~connectivity $p$ at $\varepsilon=10$, and (c)~number of agents $n$ at $p=0.5$, $\varepsilon=10$.}
        \label{fig:main}
    \end{minipage}
    \hfill
    % Right: one subfigure
    \begin{minipage}[t]{0.24\textwidth}
        \centering
        \begin{subfigure}[b]{0.99\linewidth}
            \centering
            \includegraphics[width=\linewidth]{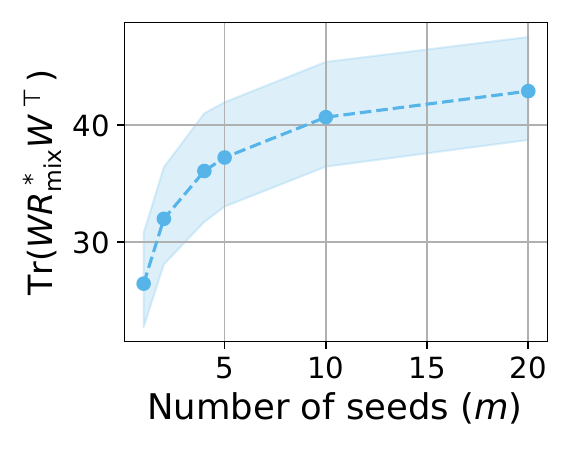}
            \caption{$n=20$, $p=0.5$, $\varepsilon=10$.}
            \label{fig:sub4}
        \end{subfigure}
        \caption{$\mathrm{Tr}(\mathbf{W}\mathbf{R}_{\text{mix}}\mathbf{W}^\top)$ vs. number of independent seeds.}
        \label{fig:main_right}
    \end{minipage}
\end{figure*}

% We now harness the insights from Section~\ref{sec:analysis} to design \emph{\algo}, a novel algorithm that integrates correlated Gaussian noise in a principled manner. Our goal is to balance two competing objectives:
% \begin{enumerate}[label=(\arabic*),leftmargin=14pt]
%     \item \textit{Utility (Convergence)}: Minimize the additive variance term $\mathrm{Tr}\!\bigl(\mathbf{W}\,\mathbf{R}\,\mathbf{W}^\top\bigr)$ in Theorem~\ref{thm:convergence}.
%     \item \textit{Privacy}:
%     Guarantee $(\varepsilon,\delta)$-DP by ensuring $[\mathbf{R}^{-1}]_{ii}$ remains sufficiently small (cf.\ Theorem~\ref{thm:privacy_dp}).
% \end{enumerate}

\subsection{Privacy-Utility Trade-Off}

Theorems~\ref{thm:convergence} and~\ref{thm:privacy_dp} together highlight a fundamental tension between utility and privacy: smaller diagonal entries $[\mathbf{R}]_{ii}$ reduce the model-update variance $\mathrm{Tr}(\mathbf{W}\mathbf{R}\mathbf{W}^\top)$ in Eq.~\eqref{eq:update_error} (expediting convergence), yet they increase $[\mathbf{R}^{-1}]_{ii}$ and thus amplify $\varepsilon$ in Eq.~\eqref{eq:thm:privacy_dp} (weaker privacy). Conversely, making $[\mathbf{R}]_{ii}$ large improves privacy but slows convergence. A careful choice of~$\mathbf{R}$ is therefore crucial to optimize this trade-off. 
From our analysis, we draw the following guidelines:

\emph{(G1) Invertibility.} A necessary condition for $(\varepsilon,\delta)$-LDP is $\mathbf{R} \succ 0$ (Theorem~\ref{thm:privacy_dp}). For this reason, we partition~$\mathbf{R}$ into an \emph{independent} noise component ($\sigma_{\text{mix}}^2 \mathbf{I}_n$), that ensures invertibility, and an arbitrary \emph{correlated} component ($\mathbf{R}_{\text{cor}}$): 
% expected to promote noise cancellation:
\begin{align}
    \mathbf{R}_{\text{mix}} = \sigma_{\text{mix}}^2 \mathbf{I}_n + \mathbf{R}_{\text{cor}}. 
    \label{eq:R_mix}
\end{align}
The covariance $\mathbf{R}_{\text{mix}}$ in~\eqref{eq:R_mix} generalizes prior approaches as special cases: \emph{(i) LDP}: $\mathbf{R}_{\text{cor}} = \mathbf{0}$; \emph{(ii) DECOR}: $\mathbf{R}_{\text{cor}} = \sigma_{\text{cor}}^2 \mathbf{L}$.
We remark $\sigma_{\text{mix}}^2>0$; if not, $\mathbf{R}_{\text{mix}}$ would be singular, and $\mathbf{v} \sim \mathcal{N}(\mathbf{0}, \mathbf{R}_{\text{mix}})$ could make $\mathbf{W}\mathbf{v} = \mathbf{0}$, i.e., no privacy noise. 

\emph{(G2) Optimization Problem.} We optimize $\mathbf{R}_{\text{mix}}$ to balance utility (Theorem~\ref{thm:convergence}) and $(\varepsilon,\delta)$-LDP privacy (Theorem~\ref{thm:privacy_dp}):
\begin{subequations}
\begin{align}
    \underset{\mathbf{R}_{\text{mix}} \succ 0}{\text{minimize}} 
    \quad & \mathrm{Tr}\bigl(\mathbf{W}\mathbf{R}_{\text{mix}}\mathbf{W}^\top\bigr) 
    \label{opt:trade_off}\\
    \text{subject to}\quad 
    & [\mathbf{R}_{\text{mix}}^{-1}]_{ii} \le \frac{\varepsilon^2}{16C^2T\log(1/\delta)},~\forall\,i \in \mathcal{V}.
    \label{opt:constraint}
\end{align}
\label{opt:both}
\end{subequations}
Problem~\eqref{opt:both} can be formulated as a convex semidefinite program (SDP). The objective is linear in $\mathbf{R}_{\text{mix}}$, and the mapping $\mathbf{R}_{\text{mix}} \mapsto \mathbf{R}_{\text{mix}}^{-1}$ is matrix-convex on $\mathbf{R}_{\text{mix}}\succ0$~\cite[§4.6.2]{boydConvexOptimization2004}. By applying Schur‐complement constraints, each bound on $[\mathbf{R}_{\text{mix}}^{-1}]_{ii}$ becomes a linear matrix inequality. Thus, Problem~\eqref{opt:both} can be solved efficiently by standard SDP solvers (e.g., MOSEK~\cite{MOSEK}).
Also, our SDP formulation is scalable, as we verify numerically with networks of up to 100 agents (Fig.~\ref{fig:sub3}).
Solving~\eqref{opt:both} yields the optimal solutions $\sigma_{\text{mix}}^{2,\star}$ and $\mathbf{R}_{\text{cor}}^\star$, from which we compute $\mathbf{R}_{\text{mix}}^\star=\sigma_{\text{mix}}^{2,\star}\mathbf{I}_n+\mathbf{R}_{\text{cor}}^\star$.

% The following proposition proves that $\mathbf{R}_{\text{mix}}^\star$ is optimal.

% \egl{Can we replace $R_{mix}> 0$ with $R_{mix}\ge 0$ in the problem formulation? Equality can never happen because of the constraint but the way it's written it doesn't look like an SDP constraint}
% Problem~\eqref{opt:both} defines a convex semidefinite program (SDP): the objective $\operatorname{Tr}(\mathbf{W} \mathbf{R} \mathbf{W}^\top)$ is linear in $\mathbf{R}$, and the mapping $\mathbf{R} \mapsto \mathbf{R}^{-1}$ is matrix convex over $\mathbf{R} \succ 0$, thus the constraints define a convex set~\cite{boydConvexOptimization2004}. 
% We find optimization variables $\sigma_{\text{mix}}^{2,\star}$ and $\mathbf{R}_{\text{cor}}^\star$ efficiently through standard SDP solvers and compute $\mathbf{R}_{\text{mix}}^\star = \sigma_{\text{mix}}^{2,\star} \mathbf{I}_n + \mathbf{R}_{\text{cor}}^\star$.
% that the optimization problem optimizes the trade-off.

% \egl{"Optimally balances" is vague, make this precise.
% Also, "corollary" usually means a direct consequence of a theorem just stated. Consider "proposition" or something like that instead.}

% Corollary~\ref{cor:trade_off} formalizes this trade-off by minimizing $\mathrm{Tr}(\mathbf{W}\,\mathbf{R}\,\mathbf{W}^\top)$, while constraining $[\mathbf{R}^{-1}]_{ii}$ to ensure $(\varepsilon,\delta)$-DP. 

\begin{proposition}
\label{cor:trade_off}
% Let $\varepsilon \le \log(1/\delta)$. Problem~\eqref{opt:both} maximizes the utility of our \emph{covariance-based} framework under $(\varepsilon,\delta)$-LDP.
% Let $\varepsilon\le\log(1/\delta)$.  
Problem~\eqref{opt:both} attains the \emph{largest} utility
among all covariances satisfying~\eqref{eq:R_mix}; every such solution is $(\varepsilon,\delta)$‑LDP.
\end{proposition}

\begin{proof}
First, the objective in~\eqref{opt:trade_off} minimizes the privacy-noise variance in Theorem~\ref{thm:convergence}, Eq.~\eqref{eq:update_error}. Second, applying Eq.~\eqref{opt:constraint} to the privacy bound in Theorem~\ref{thm:privacy_dp}, Eq.~\eqref{eq:thm:privacy_dp} shows that
our framework is ($\varepsilon^*,\delta$)-LDP, where $\varepsilon^* \leq \frac{\varepsilon^2}{8\log(1/\delta)} + \frac{\varepsilon}{\sqrt{2}} \leq \varepsilon$.
% Therefore, solving \eqref{opt:trade_off} and \eqref{opt:constraint} simultaneously (i) minimizes convergence error while (ii) preserving the $(\varepsilon,\delta)$-DP guarantee.  
\end{proof}

\subsection{Correlated Noise DSGD (\algo)}

\begin{algorithm}[t]
    \caption{\textsc{\algo}}
    \label{alg:whisper_dsgd}
    \begin{algorithmic}[1]
    \Statex \textbf{Input:} mixing matrix $\mathbf{W}$, $(\varepsilon,\delta)$-DP parameters, number of iterations $T$, clipping threshold $C$, random seed $s$.
    \State Initialize
    \(\!\mathbf{R}_{\text{mix}} = \sigma_{\text{mix}}^2 \mathbf{I}_n + \mathbf{R}_{\text{cor}}\)
    \Comment{guideline \emph{(G1)}} \label{alg:line1}
    \State Optimize~\eqref{opt:both} to find 
    \(\mathbf{R}_{\text{mix}}^\star\)  
    \Comment{guideline \emph{(G2)}} \label{alg:line2}
    \State Share 
    (\(\mathbf{R}_{\text{mix}}^\star\), $s$) among all agents \label{alg:line3}
    \For{$t$ \textbf{in} $0 \dots T-1$} \label{alg:line4}
            \State $\hat{\mathbf{g}}^{(t)} = \mathrm{clip}\bigl(\mathbf{g}^{(t)}; C\bigr)$ \Comment{clipped gradient}
            \State $\mathbf{v}^{(t)} = \mathbf{R}_{\text{mix}}^{\star 1/2}\mathbf{s}^{(t)}$, $\mathbf{s}^{(t)} \sim \mathcal{N}(\mathbf{0},\mathbf{I}_n)$ \Comment{privacy noise} \label{alg:line6}
            \State $\tilde{\mathbf{g}}^{(t)} = \hat{\mathbf{g}}^{(t)} + \mathbf{v}^{(t)}$ \Comment{perturbed gradient}
            \State $\tilde{\mathbf{x}}^{(t+1)} = \mathbf{W} \left( \tilde{\mathbf{x}}^{(t)} - \eta_t \tilde{\mathbf{g}}^{(t)} \right)$ \Comment{perturbed update}
        \EndFor \label{alg:line9}
        \State \Return $\tilde{\mathbf{x}}^{(T)}$
    \end{algorithmic}
\end{algorithm}

We now present \emph{\algo} (Algorithm~\ref{alg:whisper_dsgd}), a novel differentially private DSGD algorithm that integrates our covariance-based framework with the theoretical guidelines.

\emph{\algo}: \emph{(i)}~initializes $(\sigma_{\text{mix}}^2, \mathbf{R}_{\text{cor}})$ so that $\mathbf{R}_{\text{mix}}$ is invertible (line~\ref{alg:line1});
\emph{(ii)}~optimizes Problem~\eqref{opt:both} via standard SDP solvers (line~\ref{alg:line2});
\emph{(iii)}~shares the optimal $\mathbf{R}_{\text{mix}}^\star$ and a global random seed~$s$ among all agents (line~\ref{alg:line3});
\emph{(iv)}~executes the standard differentially-private DSGD algorithm (lines~\ref{alg:line4}--\ref{alg:line9}), except agents add privacy noise $\mathbf{v}^{(t)} = \mathbf{R}_{\text{mix}}^{\star 1/2}\mathbf{s}^{(t)}$, where $\mathbf{s}^{(t)} \sim \mathcal{N}(\mathbf{0},\mathbf{I}_n)$ is locally generated from seed $s$ (line~\ref{alg:line6}).

Once agents share $\mathbf{R}_{\text{mix}}^\star$ and seed $s$, \emph{\algo} runs fully decentralized, incurring no extra communication overhead than baseline DSGD (lines~\ref{alg:line4}--\ref{alg:line9}). However, computing $\mathbf{R}_{\text{mix}}^\star$ requires either a centralized or a distributed procedure (lines~\ref{alg:line1}--\ref{alg:line3}). In the \emph{centralized} approach, a single node with knowledge of \(\mathbf{W}\) and parameters \((\varepsilon,\delta,T,C)\) solves~\eqref{opt:both} and broadcasts $(\mathbf{R}_{\text{mix}}^\star,s)$ to all agents. Alternatively, a \emph{distributed} setup involves an initial consensus phase on $\mathbf{W}$ and $s$, then each agent locally solves~\eqref{opt:both}. This approach removes reliance on any single node but incurs an initial communication overhead.

\subsection{Comparison with Related Works}
\label{subsec:related_work}

\begin{figure*}[tbp] 
    \centering
    \hspace{5ex}
    \includegraphics[width=0.38\linewidth]{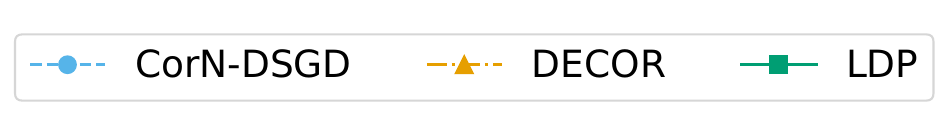}
    
    % First figure
    \begin{minipage}[t]{0.497\textwidth}
        \centering
        \begin{subfigure}[b]{0.494\linewidth}
            \centering
            \includegraphics[width=\linewidth]{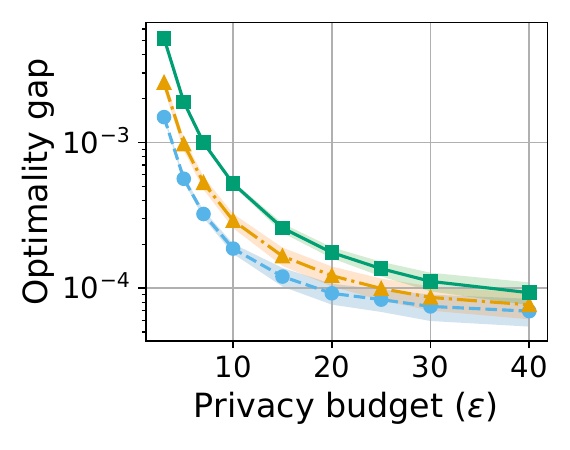}
            \caption{Synthetic, $p=0.5$.}
            \label{fig:1a}
        \end{subfigure}
        \hfill
        \begin{subfigure}[b]{0.494\linewidth}
            \centering
            \includegraphics[width=\linewidth]{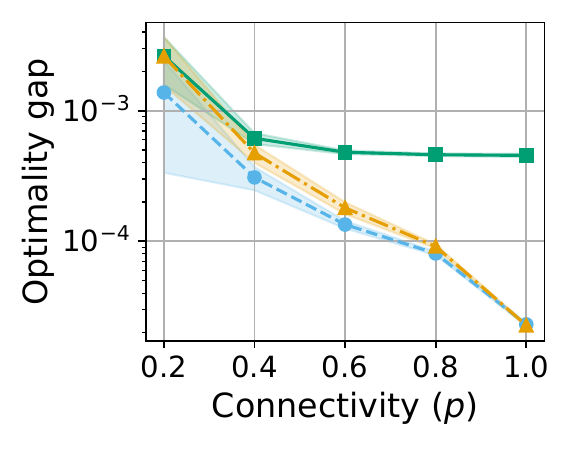}
            \caption{Synthetic, $\varepsilon=10$.}
            \label{fig:1b}
        \end{subfigure}
        % \hfill
        % \begin{subfigure}[b]{0.32\linewidth}
        %     \centering
        %     \includegraphics[width=\linewidth]{example-image-c}
        %     \caption{MNIST, NN.}
        %     \label{fig:1c}
        % \end{subfigure}
        \caption{\textbf{Quadratic Optimization.} Optimality gap versus (a) privacy budget $\varepsilon$ at $p=0.5$ and (b) connectivity $p$ at $\varepsilon=10$.}
        \label{fig:main1}
    \end{minipage}
    \hfill
    % Second figure
    \begin{minipage}[t]{0.497\textwidth}
        \centering
        \begin{subfigure}[b]{0.494\linewidth}
            \centering
            \includegraphics[width=\linewidth]{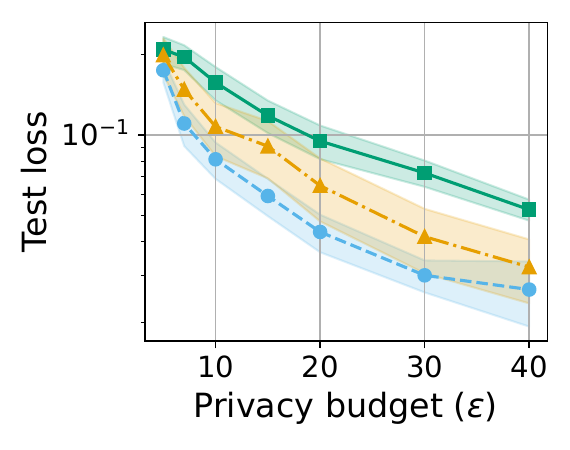}
            \caption{a9a dataset, $p=0.5$.}
            \label{fig:2a}
        \end{subfigure}
        \hfill
        \begin{subfigure}[b]{0.494\linewidth}
            \centering
            \includegraphics[width=\linewidth]{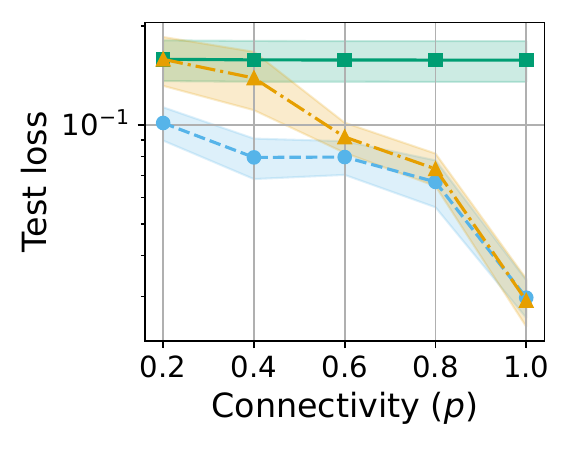}
            \caption{a9a dataset, $\varepsilon=10$.}
            \label{fig:2b}
        \end{subfigure}
        % \hfill
        % \begin{subfigure}[b]{0.32\linewidth}
        %     \centering
        %     \includegraphics[width=\linewidth]{example-image-c}
        %     \caption{MNIST, NN.}
        %     \label{fig:2c}
        % \end{subfigure}
        \caption{\textbf{Logistic Regression.} Test loss versus (a)~privacy budget~$\varepsilon$ at $p=0.5$ and (b)~connectivity $p$ at $\varepsilon=10$.}
        \label{fig:main2}
    \end{minipage}
\end{figure*}

\begin{figure}[t] 
    % Third figure
    \begin{minipage}[t]{0.497\textwidth}
        \centering
        \begin{subfigure}[b]{0.494\linewidth}
            \centering
            \includegraphics[width=\linewidth]{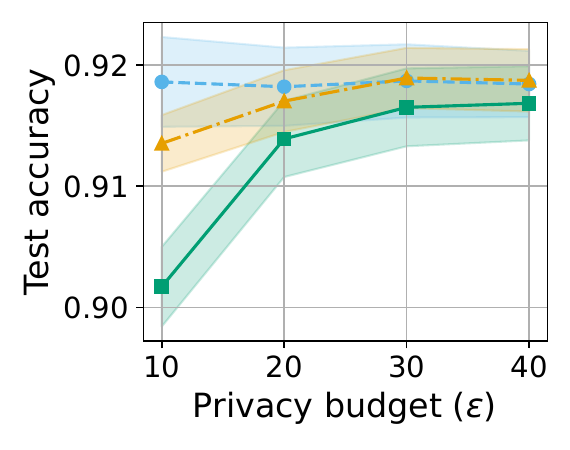}
            \caption{MNIST, $p=0.5$.}
            \label{fig:3a}
        \end{subfigure}
        \hfill
        \begin{subfigure}[b]{0.494\linewidth}
            \centering
            \includegraphics[width=\linewidth]{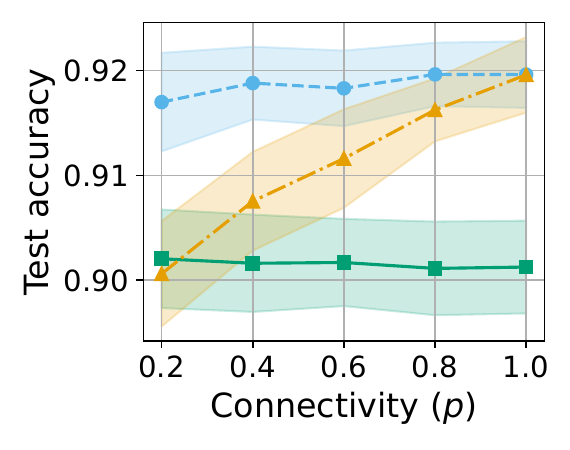}
            \caption{MNIST, $\varepsilon=10$.}
            \label{fig:3b}
        \end{subfigure}
        \caption{\textbf{Image Classification.} Test accuracy versus (a) privacy budget $\varepsilon$ at $p=0.5$ and (b) connectivity $p$ at $\varepsilon=10$.}
        \label{fig:main3}
    \end{minipage}
    \vspace{-\baselineskip}
\end{figure}

We compare $\mathrm{Tr}(\mathbf{W}\mathbf{R}_{\text{mix}}\mathbf{W}^\top)$ (the effective noise variance in \emph{\algo}) against the following state‐of‐the‐art baselines:
\emph{(i) pairwise correlation}, by replacing \(\mathbf{R}_{\text{mix}}\) in \eqref{opt:both} with \(\mathbf{R}_{\text{pair}} = \sigma_{\text{pair}}^2\mathbf{I}_n + \sigma_{\text{cor}}^2\mathbf{L}\), 
and \emph{(ii) LDP}, with \(\mathbf{R}_{\text{ldp}} = \sigma_{\text{ldp}}^2\mathbf{I}_n\).

Figure~\ref{fig:main} reports solutions of Problem~\eqref{opt:both} for $\mathbf{R}_{\text{mix}}$, $\mathbf{R}_{\text{pair}}$, and $\mathbf{R}_{\text{ldp}}$ in an Erdős--Rényi graph under varying privacy budgets \(\varepsilon\) (Fig.~\ref{fig:sub1}), network connectivity \(p\) (Fig.~\ref{fig:sub2}), and number of agents \(n\) (Fig.~\ref{fig:sub3}). 
The fully independent LDP noise imposes the highest variance $\mathrm{Tr}(\mathbf{W}\mathbf{R}_{\text{ldp}}\mathbf{W}^\top)$, as it forgoes any correlation-based cancellation.
% (it serves as a practical \emph{upper bound}).
% Conversely, \(\sigma_{\text{cdp}}^2\) acts as a theoretical \emph{lower bound}, representing the ideal CDP-level variance. 
As \(\varepsilon\) and \(p\) increase (weaker privacy requirements or denser connectivity),
correlation‐based schemes ($\mathbf{R}_{\text{pair}}$ and~$\mathbf{R}_{\text{mix}}$) become more effective, and $\mathbf{R}_{\text{mix}}$ consistently achieves lower variance than~$\mathbf{R}_{\text{pair}}$. This gap is particularly large for sparser networks (lower $p$), since \(\mathbf{R}_{\text{mix}}\) leverages broader covariance structures (through \(\mathbf{R}_{\text{cor}}\)), rather than just the graph Laplacian~\(\mathbf{L}\).

\subsection{Honest-But-Curious (HBC) Setting}
\label{subsec:hbc}

We extend Problem~\eqref{opt:both} to account for honest-but-curious (HBC) agents. 
Let $\mathcal{S} = \{ s_k \}_{k=1}^{m}$ be $m$ independent seeds. Seed $s_k$ is known to the (possibly overlapping) set of agents $\mathcal H_k \subseteq \mathcal{V}$. The covariance is decomposed as: $\mathbf{R}_{\text{cor}}=\sum_{k=1}^{m} \mathbf{R}_{\text{cor}}^{(k)}$, with $\mathrm{supp}(\mathbf{R}_{\text{cor}}^{(k)}) \subseteq \mathcal{H}_{k} \times \mathcal{H}_{k}$. Thus the pair $(\mathbf{R}_{\text{cor}}^{(k)}, s_k)$ is hidden from every agent outside $\mathcal H_k$.

For any adversarial coalition $I \subseteq \mathcal{V}$, define the set of unknown seeds $\mathcal{U}(I) \coloneqq \{ k : \mathcal{H}_{k} \cap I = \emptyset \}$. The covariance \emph{unknown} to $I$ is $\mathbf{R}_{\text{mix}}^{(I)} \coloneqq \sigma_{\text{mix}}^2 \mathbf{I}_n + \sum_{k \in \mathcal{U}(I)} \mathbf{R}_{\text{cor}}^{(k)}$. 
We protect against coalitions of size at most~$q$: $\mathcal{C}_{q} \coloneqq \{ I \subseteq \mathcal{V} : |I| \leq q \}$.

In this context, the utility analysis in Theorem~\ref{thm:convergence} remains unchanged, while Theorem~\ref{thm:privacy_dp} requires only minor adaptation, as each coalition $I \in \mathcal{C}_{q}$ imposes its own privacy constraint:
\begin{subequations}
\begin{align}
    \textstyle
    \underset{\sigma_{\text{mix}}^2 > 0, \{ \mathbf{R}_{\text{cor}}^{(k)} \}}{\text{minimize}} 
    \quad & \mathrm{Tr}\bigl(\mathbf{W}\mathbf{R}_{\text{mix}}\mathbf{W}^\top\bigr)
    \label{opt:trade_off_hbc}\\
    \text{subject to}\quad 
    & \textstyle
    \mathbf{R}_{\text{mix}} = \sigma_{\text{mix}}^2 \mathbf{I}_n + \sum_{k=1}^{m} \mathbf{R}_{\text{cor}}^{(k)}, \\
    & \textstyle 
    \mathbf{R}_{\text{cor}}^{(k)} \succeq \mathbf{0},~\mathrm{supp}(\mathbf{R}_{\text{cor}}^{(k)}) \subseteq \mathcal{H}_{k} \times \mathcal{H}_{k}, \\
    & \textstyle
    [(\mathbf{R}_{\text{mix}}^{(I)})^{-1}]_{ii} \le \frac{\varepsilon^2}{16C^2T\log(1/\delta)}, \\
    & \forall I \in \mathcal{C}_{q}, i \not \in I. \notag
    \label{opt:constraint_hbc}
\end{align}
\label{opt:both_hbc}
\end{subequations}
Problem~\eqref{opt:both_hbc} is  solvable in polynomial time. Special cases are: \emph{(i)} external eavesdropper ($q=0$, $m=1$); \emph{(ii)} HBC without collusion ($q=1$); \emph{(iii)} \emph{DECOR} ($q \geq 1$, $m=|\mathcal{E}|$); \emph{(iv)} \emph{LDP} ($q=n-1$, $m=n$).
Fig.~\ref{fig:main_right} shows the value of $\mathrm{Tr}(\mathbf{W}\mathbf{R}_{\text{mix}}\mathbf{W}^\top)$ for different numbers of independent seeds ($m\in\{1,\dots,n\}$) and disjoint groups of agents ($\bigcap_{k}\mathcal H_k=\emptyset$). 

\vspace{-0.2cm}
\section{Experimental Evaluation}

\label{sec:experiments}

% We evaluate \emph{\algo} on two tasks (quadratic optimization and real‐world logistic regression), for varying privacy budgets and network connectivity.
% Our code is available at \url{https://github.com/arodio/WhisperDSGD}.

% We evaluate \emph{\algo} on synthetic quadratic objectives and real-world logistic regression, under varying privacy budgets \(\varepsilon\) and network connectivity \(p\). Our code is available at \url{https://github.com/arodio/WhisperDSGD}.

\subsection{Experimental Setup}

\subsubsection{Network and DP parameters}
We simulate $n = 20$ agents on an Erd\H{o}s–R\'enyi graph with varying connectivity $p \in \{0.2, 0.4, 0.6, 0.8, 1.0\}$. We use Metropolis–Hastings weights: $w_{ij}=\tfrac{1}{|\mathcal{N}_i|+1}$ if $j\in\mathcal{N}_i$~\cite{boydRandomizedGossipAlgorithms2006}. For comparison with~\cite{allouah2024privacy}, we focus on external eavesdroppers observing all agents' communications ($q=0$, $m=1$), and we set~DP~parameters $\varepsilon \in \{3,5,7,10,15,20,25,30,40\}$, $\delta=10^{-5}$, $C=0.1$.

\subsubsection{Algorithms} 
We compare 
\emph{(i)}~\emph{LDP}~\cite{kasiviswanathanWhatCanWe2011, duchiLocalPrivacyStatistical2013}, which adds independent local noise;
\emph{(ii)}~\emph{DECOR}~\cite{allouah2024privacy}, representative of pairwise correlation;
and \emph{(iii)} \emph{\algo}, our covariance-based approach.
All algorithms run for $T=5000$ iterations. 
Results are averaged over 10 random seeds.
% \egl{Fixing $T$ does not show how convergence and privacy leakage depends on $T$.}

\subsubsection{Tasks}
We consider three machine learning tasks:

\noindent\textbf{(i) Quadratic Optimization} with strictly convex objectives:
\begin{align}
      f_i(x_1,x_2) = 
  \begin{cases}
    15(x_1+i)^2 + x_2^2, & i=1,\dots,\tfrac{n}{2}, \\
    15(x_1-i)^2 + x_2^2, & i=\tfrac{n}{2}+1,\dots,n.
  \end{cases}
\end{align}
To introduce data heterogeneity, we rotate $f_i$, $i=\tfrac{n}{2}+1,\dots,n$ around the local minimizers $(i,0)$ by an angle $\theta=15^{\circ}$~\cite{alshedivatFederatedLearningPosterior2020}.
We use a diminishing step size \(\eta_t = 1/\sqrt{t}\), with \(\eta_1=10^{-2}\).

\noindent\textbf{(ii) Logistic Regression.} We train a regularized linear model on the a9a LIBSVM dataset~\cite{changLIBSVMLibrarySupport2011}, with an 80-20 train-test split. 
We simulate data heterogeneity by randomly partitioning samples among agents using a Dirichlet distribution with parameter 10~\cite{allouah2024privacy}.
We tune the step-size in $\{10^{-1},5\cdot10^{-2},10^{-2},5\cdot10^{-3},10^{-3}\}$ and fix batch size to 128.

\noindent\textbf{(iii) Image Classification.} We train a shallow neural network on the MNIST dataset~\cite{lecun-mnisthandwrittendigit-2010}, with the same train-test split, data partitioning scheme, and hyperparameter grid as in task~(ii).

\subsection{Experimental Results}

Figures~\ref{fig:main1}--\ref{fig:main3} show the empirical privacy-utility trade-off on (i) the quadratic optimization task (reporting optimality gap $\frac{1}{n}\!\sum_{i=1}^{n} f_i(x_i^{(T)}) - f^*$), (ii) logistic regression task (reporting test loss), and (iii) image classification task (reporting test accuracy). Panel (a) of each figure varies the privacy budget~\(\varepsilon\) under fixed connectivity \(p=0.5\), while panel (b) varies \(p\) under fixed \(\varepsilon=10\).
Across all \(\varepsilon\) values, both correlation-based approaches (\emph{DECOR} and \emph{\algo}) outperform \emph{LDP} (Figs.~\ref{fig:main1}--\ref{fig:3a}).
\emph{\algo} consistently outperforms \emph{DECOR} for all network connectivities (Figs.~\ref{fig:main1}--\ref{fig:3b}). For fully connected networks (\(p=1\)), \(\mathbf{R}_{\text{cor}}\) and \(\mathbf{L}\) become proportional, therefore both methods achieve nearly the same utility. However, at lower connectivities (\(p \le 0.4\)), \emph{\algo} outperforms \emph{DECOR} by the largest margin, confirming the benefits of our covariance-based framework compared to prior approaches in sparse networks.

\section{Conclusion}

In serverless decentralized learning, sensitive information about training data can be exposed through shared local models. While LDP mechanisms mitigate such privacy threats, they suffer from reduced utility due to independently added local privacy noise.
We propose the \emph{\algo} algorithm, which generates correlated noise across agents using a covariance-based approach. This method incorporates the mixing weights into the design of correlated noise, offering strong privacy guarantees with minimal impact on utility. Both theoretical analysis and empirical results demonstrate that \emph{\algo} improves the privacy-utility trade-off compared to existing methods (\emph{LDP}, \emph{DECOR}) under various privacy budgets, particularly in sparse network topologies.

%%%%%%
%% Appendix:
%% If needed a single appendix is created by
%%
%\appendix
%%
%% If several appendices are needed, then the command
%%
% \appendices
%%
%% in combination with further \section commands can be used.
%%%%%%

%\section*{Acknowledgment}

%%%%%%
%% References:
%% We recommend the usage of BibTeX:
%%
\clearpage
\bibliographystyle{IEEEtran}
\bibliography{references}
%%
%% where we here have assumed the existence of the files
%% definitions.bib and bibliofile.bib.
%% BibTeX documentation can be obtained at:
%% http://www.ctan.org/tex-archive/biblio/bibtex/contrib/doc/
%%%%%%

\end{document}